\documentclass[]{article}

\usepackage[letterpaper]{geometry}

\usepackage[latin1]{inputenc}
\usepackage{setspace}
\usepackage{fancyhdr}
\usepackage{tocloft}
\usepackage[usenames]{color}
\RequirePackage[OT1]{fontenc}
\usepackage{amsmath,amsfonts,amssymb,amsthm}
\usepackage{graphicx}
\usepackage{tikz}
\usetikzlibrary{arrows}
\usepackage{caption}
\usepackage{subcaption}
\usepackage{paralist}
\usepackage{amsbsy}

\usepackage[mathscr]{eucal}

\usepackage{stmaryrd}

\usepackage{bm}

\usepackage{xspace}

\usepackage{algorithm}
\usepackage{algpseudocode}
\usepackage{comment}

\mathchardef\mhyphen="2D

\ifpdf
\usepackage[colorlinks,urlcolor=blue,citecolor=blue,linkcolor=blue]{hyperref}
\else
\newcommand{\href}[2]{{#2}}
\usepackage{url}
\fi

\ifpdf
\newcommand{\Sec}[1]{\hyperref[sec:#1]{\S\ref*{sec:#1}}} 
\newcommand{\App}[1]{\hyperref[sec:#1]{Appendix~\ref*{sec:#1}}} 
\newcommand{\Eqn}[1]{\hyperref[eq:#1]{{\rm (\ref*{eq:#1})}}} 
\newcommand{\Part}[1]{\hyperref[part:#1]{(\ref*{part:#1})}} 
\newcommand{\Fig}[1]{\hyperref[fig:#1]{Figure~\ref*{fig:#1}}} 
\newcommand{\Tab}[1]{\hyperref[tab:#1]{Table~\ref*{tab:#1}}} 
\newcommand{\Thm}[1]{\hyperref[thm:#1]{Theorem~\ref*{thm:#1}}} 
\newcommand{\Lem}[1]{\hyperref[lem:#1]{Lemma~\ref*{lem:#1}}} 
\newcommand{\Prop}[1]{\hyperref[prop:#1]{Proposition~\ref*{prop:#1}}} 
\newcommand{\Cor}[1]{\hyperref[cor:#1]{Corollary~\ref*{cor:#1}}} 
\newcommand{\Def}[1]{\hyperref[def:#1]{Definition~\ref*{def:#1}}} 
\newcommand{\Alg}[1]{\hyperref[alg:#1]{Algorithm~\ref*{alg:#1}}} 
\newcommand{\Ex}[1]{\hyperref[ex:#1]{Example~\ref*{ex:#1}}} 
\newcommand{\As}[1]{\hyperref[as:#1]{Assumption~{\rm\ref*{as:#1}}}} 
\newcommand{\Reg}[1]{\hyperref[as:#1]{Condition~\ref*{reg:#1}}} 
\newcommand{\AlgLine}[2]{\hyperref[alg:#1]{line~\ref*{line:#2} of Algorithm~\ref*{alg:#1}}}
\newcommand{\AlgLines}[3]{\hyperref[alg:#1]{lines~\ref*{line:#2}--\ref*{line:#3} of Algorithm~\ref*{alg:#1}}}
\else
\newcommand{\Sec}[1]{{\S\ref{sec:#1}}} 
\newcommand{\App}[1]{{Appendix~\ref{sec:#1}}} 
\newcommand{\Eqn}[1]{{(\ref{eq:#1})}} 
\newcommand{\Part}[1]{{(\ref{part:#1})}} 
\newcommand{\Fig}[1]{{Figure~\ref{fig:#1}}} 
\newcommand{\Tab}[1]{{Table~\ref{tab:#1}}} 
\newcommand{\Thm}[1]{{Theorem~\ref{thm:#1}}} 
\newcommand{\Lem}[1]{{Lemma~\ref{lem:#1}}} 
\newcommand{\Prop}[1]{{Proposition~\ref{prop:#1}}} 
\newcommand{\Cor}[1]{{Corollary~\ref{cor:#1}}} 
\newcommand{\Def}[1]{{Definition~\ref{def:#1}}} 
\newcommand{\Alg}[1]{{Algorithm~\ref{alg:#1}}} 
\newcommand{\Ex}[1]{{Example~\ref{ex:#1}}} 
\newcommand{\Reg}[1]{{R~\ref*{reg:#1}}} 
\fi

\newcommand{\Real}{\mathbb{R}}
\newcommand{\dom}{{\bf dom}\,}

\newcommand{\Tra}{^{\sf T}} 
\newcommand{\tr}{\operatorname{tr}} 

\newcommand{\V}[1]{{\bm{\mathbf{\MakeLowercase{#1}}}}} 
\newcommand{\Vtilde}[1]{{\bm{\tilde \mathbf{\MakeLowercase{#1}}}}} 

\newcommand{\M}[1]{{\bm{\mathbf{\MakeUppercase{#1}}}}} 

\newcommand{\Mhat}[1]{{\bm{\hat \mathbf{\MakeUppercase{#1}}}}} 
\newcommand{\Mtilde}[1]{{\bm{\tilde \mathbf{\MakeUppercase{#1}}}}} 

\begin{document}

\newtheorem{proposition}{Proposition}[]
\newcommand\relatedversion{}
\renewcommand\relatedversion{\thanks{The full version of the paper can be accessed at \protect\url{https://arxiv.org/abs/1902.09310}}} 

\title{Towards Tuning-Free Minimum-Volume \\ Nonnegative Matrix Factorization}

\author
{Duc Toan Nguyen \thanks{Texas Christian University, Fort Worth, TX, USA (duc.toan.nguyen@tcu.edu)} \and Eric C. Chi \thanks{Department of Statistics, Rice University, Houston, TX, USA (echi@rice.edu)} }
\date{}
\maketitle
\begin{abstract}
Nonnegative Matrix Factorization (NMF) is a versatile and powerful tool for discovering  latent structures in data matrices, with many variations proposed in the literature. Recently, Leplat et al.\@ (2019) introduced a minimum-volume NMF for the identifiable recovery of rank-deficient matrices in the presence of noise. The performance of their formulation, however, requires the selection of a tuning parameter whose optimal value depends on the unknown noise level. In this work, we propose an alternative formulation of minimum-volume NMF inspired by the square-root lasso and its tuning-free properties. Our formulation also requires the selection of a tuning parameter, but its optimal value does not depend on the noise level. To fit our NMF model, we propose a majorization-minimization (MM) algorithm that comes with global convergence guarantees. We show empirically that the optimal choice of our tuning parameter is insensitive to the noise level in the data. 
    
\end{abstract}


\section{Introduction}

Nonnegative matrix factorization (NMF) is a factorization model where both the original and factor matrices are nonnegative. A matrix $\M{A}$ is nonnegative, denoted $\M{A} \geq  \V{0}$, if all of its entries are nonnegative. In particular, given a matrix $\M{X} \in \mathbb{R}^{m\times n}$, $\M{X} \geq \V{0}$, and a target rank $r$ such that $0< r \ll \min(m, n)$, NMF aims to find a matrix factorization model 
\begin{eqnarray*}
 \M{X} & \approx & \M{W}\M{H},  
\end{eqnarray*}
where $\M{W} \in \mathbb{R}^{m\times r}$ and $\M{H} \in \mathbb{R}^{r\times n}$ are nonnegative factor matrices. In practice, many data may be represented as nonnegative matrices, for example hyperspectral images and text. In the former, NMF can be used to identify pure pixel types and their relative abundances in an image. In the latter, NMF can be used to learn clusters of topics from a collection of documents. See \cite{fu2019nonnegative} for the wide array of applications of NMF.

In a noiseless setting where there are nonnegative factor matrices $\M{W}$ and $\M{H}$ such that $\M{X} = \M{W}\M{H}$, Fu et al.\@ proved in \cite{Fu2015} that under suitable regularity conditions on $\M{W}$ and $\M{H}$, any solution to the following optimization problem will recover the true underlying factor matrices up to permutation and scaling ambiguities
\begin{equation}
\label{eq:minvol_noiseless}
\begin{aligned}
\underset{\M{W}, \M{H}}{\min} \quad & \det\left(\M{W}\Tra\M{W}\right)\\
\textrm{s.t.} \quad & \M{X} = \M{W}\M{H},\, \M{H} \geq 0 \text{ and } \V{1}\Tra \M{H}  = \V{1}\Tra.
\end{aligned}
\end{equation}
The objective function in (\ref{eq:minvol_noiseless}) is proportional to the volume of the convex hull of the columns of $\M{W}$. Thus, the above problem seeks a minimum-volume (min-vol) NMF of $\M{X}$. Similar results were shown in parallel work in \cite{Lin2015}.

Recently, Leplat et al.\@ \cite{Leplat2019} and Thanh et al.\@ \cite{thanh2021inertial} proposed natural extensions to handle the case where there is noise. Specifically, Leplat et al.\@ posed the following optimization problem for noisy min-vol NMF
\begin{equation}
\begin{aligned}
\min_{\M{W}, \M{H}} \quad & \lVert \M{X}-\M{W}\M{H} \rVert_{\text{F}}^2 + \lambda  \log \det (\M{W}\Tra\M{W}+\delta \M{I}),\\
\textrm{s.t.} \quad & \M{W}, \M{H} \geq 0 \text{ and } \V{1}\Tra \M{H}(:,j) \leq 1, \forall j,
\end{aligned}
 \label{eq:gillis_mv}
\end{equation}
where $\lVert \cdot \rVert_\text{F}$ is the Frobenius norm, $\M{H}(:,j)$ is the $j$th column of matrix $\M{H}$, and $\delta$ is a positive constant introduced to allow $\M{W}$ to be rank-deficient.
The positive tuning parameter $\lambda$ plays an important role in trading off data fit as quantified by the least squares data fidelity term $\lVert \M{X} - \M{W}\M{H} \rVert_{\text{F}}^2$ and model complexity as quantified by the log-determinant penalty term. For notational convenience, we denote $\V{\theta} = (\M{W},\M{H})$ and the constraint set
\begin{eqnarray*}
\mathcal{S} & = & \lbrace (\M{W},\M{H}) : \M{W}, \M{H} \geq \M{0} \text{ and } \textbf{1}\Tra \M{H}(:,j) \leq 1, \forall j \rbrace.
\end{eqnarray*}

The optimal values of $\M{W}$ and $\M{H}$ in \Eqn{gillis_mv} correspond to the penalized maximum likelihood estimator under the assumption that 
\begin{eqnarray}
    \M{X} & = & \M{W}\M{H} + \M{E}
\end{eqnarray}
and $\M{E}$ is a matrix whose entries are i.i.d.\@ samples from a Gaussian distribution with mean zero and variance $\sigma^2$. The log-determinant penalty term corresponds to a prior incentivizing the recovery of a low complexity factor matrix $\M{W}$ with respect to its volume.

To solve \Eqn{gillis_mv}, Leplat et al.\@  proposed a block coordinate descent algorithm based on the projected fast gradient method \cite{gillis2014successive}. They set the initial value of their algorithm $(\M{W}_0,\M{H}_0)$ with the solution to the successive nonnegative projection algorithm (SNPA) \cite{gillis2014successive}. They further proposed the following formula to choose $\lambda$:
\begin{equation}
\label{eq:lambda_original}
    \lambda =  \Tilde{\lambda} \dfrac{\lVert \M{X}-\M{W}_0\M{H}_0 \rVert_{\text{F}}^2}{\log \det(\M{W}_0\Tra \M{W}_0 + \delta \M{I})},
\end{equation}
where $\Tilde{\lambda}$ is recommended to be between $10^{-3}$ and 1. As we will see shortly in Section 3, the quality of the NMF solution, however, depends on the noise level quantified by the unknown variance $\sigma^2$. Choosing the best $\Tilde{\lambda}$ for an unknown noise level $\sigma$ is not straightforward.  Follow up work by Thanh et al.\@ \cite{thanh2021inertial} proposed a modestly different constraint set $\mathcal{S}$, and a significantly faster algorithm than the one by Leplat et al.\@, but their formulation still requires setting a tuning parameter whose optimal value depends on the noise level.


The issue of choosing the tuning parameter $\lambda$ in \Eqn{gillis_mv} motivates this work. Specifically, we present a new problem formulation for noisy min-vol NMF and a Majorization-Minimization (MM) algorithm to iteratively compute its solution.  We demonstrate that our new min-vol NMF method can accurately recover an underlying NMF in the presence of noise under a relatively wide range of tuning parameter values.

Before we continue, we briefly review the notation used in this paper. We denote scalars by lowercase letters ($c$), vectors by lowercase boldface letters ($\V{v}$), and matrices by uppercase boldface letters ($\M{X}$). We denote the transpose matrix for $\M{X}$ by  $\M{X}\Tra$ and the $j$th column of $\M{X}$ by $\M{X}(:,j)$. We denote the vector whose entries are all one by $\V{1}$. 
In the context of optimization problems, we denote the ground-truth matrix $\M{W}^{\ast}$ 
and the matrix estimated by an algorithm $\Mhat{W}$.

\section{Tuning parameter dependence on noise level}
The following experiment illustrates the relationship between an empirically optimal tuning parameter $\lambda$ and unknown noise variance $\sigma^2$ in \Eqn{gillis_mv}. This dependency motivates our new  NMF method.

We revisit an experiment from \cite{Leplat2019}. Data are generated from ground-truth matrices $\M{W}^{\ast}$ and $\M{H}^{\ast}$ where 
\begin{eqnarray*}
    \M{W}^{\ast} & = & \begin{pmatrix} 1 & 1 & 0 & 0 \\ 0 & 0 & 1 & 1 \\ 0 & 1 & 1 & 0 \\ 1 & 0 & 0 & 1 \end{pmatrix},
\end{eqnarray*}
and $\M{H}^{\ast}$ is a random $4 \times 500$ matrix whose rows are independently drawn from a Dirichlet distribution. Subsequently, we construct a noiseless data matrix $\M{X}^{\ast} = \M{W}^{\ast} \M{H}^{\ast}$. A noise matrix $\M{E}$, whose entries are i.i.d.\@ Uniform$[0, \sigma]$ is added to $\M{X}^{\ast}$, resulting in $\M{X} = \M{X}^{\ast} + \M{E}$. Following the recommendation in \cite{Leplat2019}, we set $\delta = 0.1$. 
\begin{figure}[H]
    \centering
    \begin{subfigure}{.5\columnwidth}
  \centering
  \includegraphics[width=.9\linewidth]{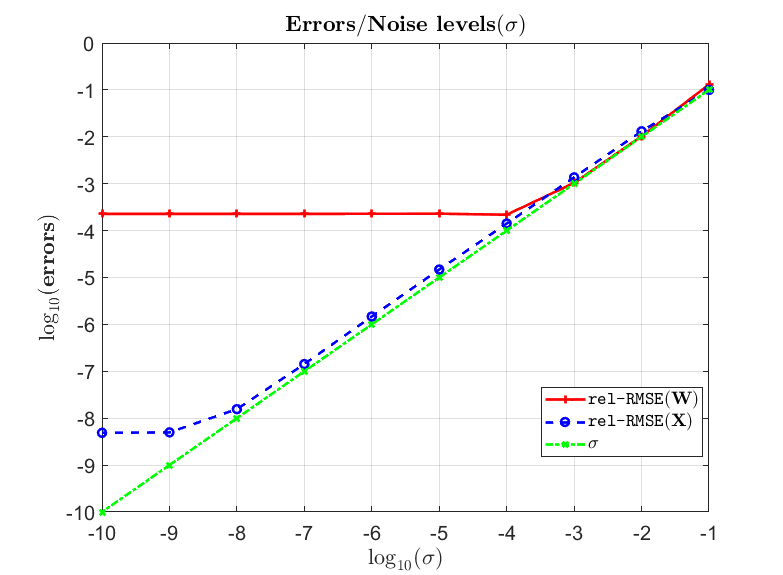}
  \caption{Smallest errors}
  \label{fig:sfig1}
\end{subfigure}%
\begin{subfigure}{.5\columnwidth}
  \centering
  \includegraphics[width=.9\linewidth]{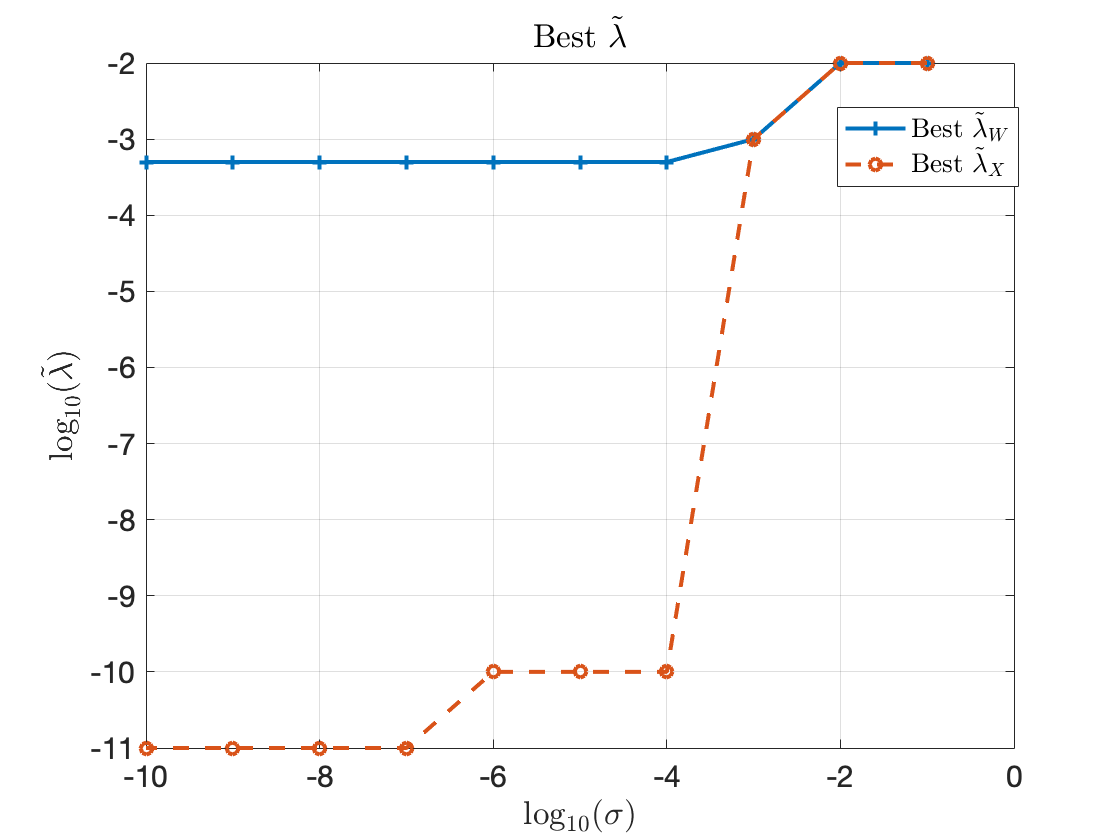}
  \caption{Best $\Tilde{\lambda}$}
  \label{fig:sfig2}
\end{subfigure}
    \caption{Results of noisy min-vol using ideally chosen $\lambda$ for different noise levels}
    \label{fig:minvol}
\end{figure}
We use the relative root-mean-squared errors between the true unobserved matrices and their estimates to evaluate the quality of a recovered NMF
\begin{eqnarray*}
    \texttt{rel-RMSE}(\M{X}) & = &\lVert\M{X}^{\ast
    }-\hat{\M{X}}\rVert_{\text{F}}/\lVert \M{X}^{\ast} \rVert_{\text{F}} \quad \text{and} \\
    \texttt{rel-RMSE}(\M{W}) & = &\lVert \M{W}^{\ast}-\hat{\M{W}} \rVert _{\text{F}}/\lVert\M{W}^{\ast}\rVert_\text{F},
\end{eqnarray*}
where $(\hat{\M{W}},\hat{\M{H}})$ are estimates of the underlying factor matrices and $\hat{\M{X}} = \hat{\M{W}}\hat{\M{H}}$. 

To investigate the relationship between an ideally chosen $\lambda$ and $\sigma$, we add different amounts of noise by varying $\sigma$ from $10^{-10}$ to $10^{-1}$. In each experiment, we run the algorithms with $\Tilde{\lambda}$ in \Eqn{lambda_original} varying from $1.5$ to $10^{-11}$. Then, we record the best results in that experiment, i.e., the smallest $ \texttt{rel-RMSE}(\M{X})$, the smallest $ \texttt{rel-RMSE}(\M{W})$, the best $\Tilde{\lambda}$ for $ \texttt{rel-RMSE}(\M{X})$ ($\Tilde{\lambda}_X$), and the best $\Tilde{\lambda}$ for $ \texttt{rel-RMSE}(\M{W})$ ($\Tilde{\lambda}_W$).  

\Fig{sfig1} shows that as $\sigma$ decreases from $10^{-1}$ to $10^{-10}$, $ \texttt{rel-RMSE}(\M{W})$ rapidly decreases to $10^{-4}$ and $ \texttt{rel-RMSE}(\M{X})$ to below $10^{-8}$. \Fig{sfig2} reveals that $\Tilde{\lambda}_W$ initially decreases with $\sigma$ before plateauing at approximately $10^{-3}$. By contrast, $\Tilde{\lambda}_X$ decreases drastically, eventually converging to $10^{-11}$ for very small values of $\sigma$.  In summary, the empirically optimal $\Tilde{\lambda}$ decreases as $\sigma$ decreases. Moreover, it is problematic that $\Tilde{\lambda}_W$ does not decrease to zero as $\sigma$ tends to zero. The issues revealed in this simulation motivate reformulating the noisy min-vol NMF problem to be less sensitive to the unknown noise level $\sigma$.



\section{Square-Root Min-Vol NMF}

We next detail our new formulation of the min-vol NMF problem.  We first propose an alternative data-fidelity term inspired by the square-root lasso. We will see that replacing the squared Frobenius norm with the Frobenius norm empirically eliminates the dependency of the optimal $\lambda$ value on the unknown noise variance $\sigma^2$. We then use a smooth approximation to the Frobenius norm which yields an optimization problem that we solve with an MM algorithm.


\subsection{The Square-Root Lasso} 

We briefly review the lasso and square-root lasso \cite{belloni2011square, Sun2012} to motivate how we pose our NMF problem. Consider a linear regression model 
\begin{eqnarray*}
    y_i & = & \V{x}_i\Tra\V{\beta}^\star + \sigma \epsilon_i \quad i = 1, \cdots, n,
\end{eqnarray*}
where the $\epsilon_i$'s are independent and identically distributed standard normal random variables. In many applications where $\V{\beta}^\star \in \Real^p$ is high-dimensional, we often assume that $\V{\beta}^\star$ is sparse in the sense that most of its entries are zero. The lasso \cite{tibshirani1996regression} estimates $\V{\beta}^\star$ by solving the convex optimization problem
\begin{eqnarray*}
    \underset{\V{\beta}}{\min}\; 
    \frac{1}{n} \lVert \V{y} - \M{X}\V{\beta} \rVert_2^2  
    + \frac{\lambda}{n}\lVert \V{\beta} \rVert_1,
\end{eqnarray*}
where $\lVert \V{\beta} \rVert_2$ and $\lVert \V{\beta} \rVert_1$ denote the 2-norm and 1-norm of $\V{\beta}$ respectively. The lasso estimator attains near optimal error in recovering $\V{\beta}^\star$ in a 2-norm sense by setting $\lambda$ to be proportional to the unkown value of $\sigma$.

By contrast, the square-root lasso estimates $\V{\beta}^\star$ by solving the convex optimization problem 
\begin{eqnarray*}
    \underset{\V{\beta}}{\min}\; 
    \sqrt{\frac{1}{n} \lVert \V{y} - \M{X}\V{\beta} \rVert_2^2} 
    + \frac{\lambda}{n}\lVert \V{\beta} \rVert_1.
\end{eqnarray*}
The square-root lasso also attains near optimal error in recovering $\V{\beta}^\star$ in a 2-norm sense but, remarkably, is able to achieve this with a value of $\lambda$ that is independent of the unknown value of $\sigma$.

Inspired by the square-root lasso, we propose a new optimization problem for min-vol NMF as
\begin{equation}
\begin{aligned}
\min_{\M{W}, \M{H}} \quad & f(\M{W},\M{H}) = \lVert \M{X}-\M{W}\M{H} \rVert_{\text{F}} + \lambda \text{vol}(\M{W})\\
\textrm{s.t.} \quad & (\M{W},\M{H}) \in \mathcal{S},
\end{aligned}
\label{squarerootprob}
\end{equation}
 where $\text{vol}(\M{W}) = \log \det(\M{W}\Tra \M{W} + \delta \M{I})$. The best penalty level $\lambda$ is expected to be independent from the variance of the noise level. We will see in our numerical studies that while some care is needed in choosing a $\lambda$ value, a very wide range of $\lambda$ values that are independent of the unknown noise level work equally well. 

One computational challenge with problem (\ref{squarerootprob}), however, is the non-differentiability of the square root function. 
Consequently, we employ a differentiable approximation of $f(\M{W}, \M{H})$ and solve the problem
\begin{equation}
\begin{aligned}
\min_{\M{W}, \M{H}} \quad & f_{\varepsilon}(\M{W},\M{H})= {\sqrt{\lVert \M{X}-\M{W}\M{H} \rVert_{\text{F}}^2 + \varepsilon}} + \lambda \text{vol}(\M{W}) \\
\textrm{s.t.} \quad & (\M{W},\M{H}) \in \mathcal{S},
\end{aligned}
\label{eq:fepsilon}
\end{equation}
where $\varepsilon$ is a small positive number. 
Subsequently, we will derive an MM method that can solve problem \Eqn{fepsilon} in the next subsections.

\subsection{The MM Principle}

The MM principle \cite{Lange2016} converts minimizing a challenging function $f(\V{\theta})$ into solving a sequence of simpler optimization problems. 
The idea is to approximate an objective function $f(\V{\theta})$ to be minimized by a surrogate function or majorization $g(\V{\theta} \mid \Vtilde{\theta})$ anchored at the current estimate $\Vtilde{\theta}$. The majorization  $g(\V{\theta} \mid \Vtilde{\theta})$ needs to satisfy two conditions: (i) a tangency condition $g(\Vtilde{\theta} \mid \Vtilde{\theta}) = f(\Vtilde{\theta})$ for all $\Vtilde{\theta}$ and (ii) a domination condition $g(\V{\theta} \mid \Vtilde{\theta}) \geq f(\Vtilde{\theta})$ for all $\V{\theta}$. The associated MM algorithm is defined by the iterates
\begin{eqnarray}
\label{eq:MM-iterate}
    \V{\theta}_{k+1} & = & \operatorname*{arg\,min}_{\V{\theta}}  g(\V{\theta} \mid \V{\theta}_{k}), ~~~ k=0, 1, \dots.
\end{eqnarray}
The tangency and domination conditions imply that
\begin{eqnarray*}
    f(\V{\theta}_{k+1}) ~~ \leq ~~ g(\V{\theta}_{k+1} \mid \V{\theta}_{k}) ~~  \leq ~~ g(\V{\theta}_{k} \mid \V{\theta}_{k}) ~~ = ~~ f(\V{\theta}_{k}).
\end{eqnarray*}
In other words, the sequence of objective function values of the MM iterates decreases monotonically. The key to successfully applying the MM principle is to construct a surrogate function that is easy to minimize.

\subsection{An MM algorithm for the approximate square-root min-vol NMF} We apply the MM principle to derive an algorithm for computing the solution to \Eqn{fepsilon}. We first derive a surrogate function that majorizes $f_\varepsilon(\M{W}, \M{H})$.

\begin{proposition}
\label{prop:majorization}
Let $\V{\theta} = (\M{W}, \M{H})$ and $\V{\theta}_k = (\M{W}_k, \M{H}_k)$. 
The following is a majorization of $f_\varepsilon(\V{\theta})$ at $\V{\theta}_k$.
\begin{equation}
    \label{eq:majorization}
    \begin{aligned}
        g(\V{\theta}\mid \V{\theta}_k) = &  \text{ }\sqrt{r_k} + \dfrac{1}{2\sqrt{r_k}}(\lVert \M{X}-\M{W}\M{H} \rVert_{\text{F}}^2+\varepsilon - r_k) \\
        & + \lambda\left[\log \det (\M{Q}_k) + \tr\left(\M{Q}_k^{-1}(\M{Q}-\M{Q}_k)\right)\right],
    \end{aligned}
\end{equation}
where $r_k = \lVert \M{X}-\M{W}_k \M{H}_k\rVert_{\text{F}}^2 + \varepsilon$, $\M{Q} = \M{W}\Tra \M{W}+\delta \M{I}$, and $\M{Q}_k = \M{W}_k\Tra \M{W}_k + \delta \M{I}$.
\end{proposition}

\begin{proof}
Recall that the first-order Taylor approximation of a differentiable concave function $\varphi(x)$ at a point $\tilde{x}$ provides a tight global upper bound on $\varphi(x)$:
\begin{eqnarray*}
\varphi(x) & \leq & \varphi(\tilde{x}) + \varphi'(x)(x - \tilde{x})
\end{eqnarray*}
for all $x, \tilde{x} \in \dom\, \varphi$. In other words, the first-order Taylor approximation of $\varphi(x)$ at $\tilde{x}$ majorizes $\varphi(x)$ at $\tilde{x}$. 

Note that $u \mapsto \sqrt{u}$ is a differentiable concave function over the positive reals and $\M{Q} \mapsto \log \det {\M{Q}}$ is a differentiable concave function over the positive definite matrices. Consequently, we can obtain the following two majorizations of 
$u \mapsto \sqrt{u}$ and $\M{Q} \mapsto \log\det \M{Q}$ at $\Vtilde{u}$.
\begin{eqnarray}
    \sqrt{u} & \leq & \sqrt{\Tilde{u}} + \dfrac{1}{2\sqrt{\Tilde{u}}}(u-\Tilde{u})
    \label{eq:sqrt(u)}
\end{eqnarray}
and 
\begin{equation}
   \log \det(\M{Q}) \leq  \log \det(\Mtilde{Q}) + \tr\left(\Mtilde{Q}^{-1}(\M{Q}-\Mtilde{Q})\right).
    \label{eq:logdetQ}
\end{equation}
Plugging 
\begin{eqnarray*}
     u & = & \lVert \M{X} - \M{W}\M{H} \rVert_{\text{F}}^2+\varepsilon, \\
    \Tilde{u} & = & r_k,\\
    \Mtilde{Q} & = & \M{Q}_k,
\end{eqnarray*}
into \Eqn{sqrt(u)} and \Eqn{logdetQ} completes the proof. 

\end{proof}
Minimizing the majorization \Eqn{majorization} gives us the following MM-update
\begin{eqnarray}
    \V{\theta}_{k+1} 
    & = & \underset{(\M{W},\M{H}) \in \mathcal{S}}{\arg\min} g(\M{W},\M{H} \mid \M{W}_k,\M{H}_k) \\
    & = & \underset{(\M{W},\M{H}) \in \mathcal{S}}{\arg\min} \lVert\M{X} - \M{W}\M{H}\rVert_{\text{F}}^2 + \lambda_k \tr\left((\M{Q}_k^{-1})\M{W}\Tra\M{W}\right), 
    \label{algomap}
\end{eqnarray}

where $\lambda_k = 2\sqrt{r_k} \lambda$. 
\\

The majorization in (\ref{algomap}) is identical in form to the majorization employed in Leplat et al.\@ \cite{Leplat2019}. The objective function $g(\M{W}, \M{H} \mid \M{W}_k, \M{H}_k)$ is biconvex, i.e., convex in $\M{W}$ with $\M{H}$ fixed and convex in $\M{H}$ with $\M{W}$ fixed. 
Therefore, we use the min-vol NMF algorithm in \cite{Leplat2019} to compute the MM-update (\ref{algomap}). \Alg{cap} summarizes our new NMF procedure.
\begin{algorithm}[H]
\caption{Square-Root Min-Vol NMF}\label{alg:cap}
 \hspace*{\algorithmicindent} \textbf{Input}: $\M{X} \in \mathbb{R}_+^{m \times n}$, target rank $r$, $\lambda_1$, $\delta$, $\varepsilon$. \\
 \hspace*{\algorithmicindent} \textbf{Output}: $\M{W} \in \mathbb{R}_+^{m \times r},\M{H} \in \mathbb{R}_+^{r \times n}$ in $\mathcal{S}$ 
\begin{algorithmic}[1]
\State ($\M{W}_1$,$\M{H}_1$) = \textbf{SNPA}($\M{X}$,$r$)
\For{$k=1, \ldots$}
    \State ($\M{W}_{k+1}$,$\M{H}_{k+1}$) = \textbf{MinVol}($\M{X}$,$r$,[$\M{W}_k$,$\M{H}_k$,$\lambda_{k}$,$\delta$]) 
    \State $\lambda_{k+1} \gets (2\sqrt{\lVert\M{X} - \M{W}\M{H}\rVert_{\text{F}}^2 + \varepsilon} ) \lambda_k$
\EndFor
\end{algorithmic}
\end{algorithm}

Note that \textbf{MinVol} is the min-vol NMF algorithm \cite{Leplat2019}. Overall, \Alg{cap} can deal with the problem of selecting the optimal tuning parameter $\lambda$ with respect to noise level better than original min-vol algorithm. Moreover, it comes with the following global convergence guarantee.
\begin{proposition}
 \label{prop:convergence}
    The limit points of the iterate sequence produced by \Alg{cap} are first order stationary points of \Eqn{fepsilon}.
\end{proposition}
The proof is given in the supplement.


\section{Numerical experiments}
In this section, we investigate the empirical performance and behavior of the square-root min-vol NMF. 
We conduct three experiments. In the first experiment, we revisit the experiment from Section~2 to see that a single $\lambda$ value can recover, up to the noise level, the underlying factor matrices over a dynamic range of noise levels. This is in contrast to the behavior seen with the Leplat et al.\@ min-vol NMF formulation in Figure~1. In the second experiment, we further explore the sensitivity of the choice of $\lambda$ as a function of noise. In the third experiment, we take a closer look at the behavior of the parameter $\lambda_k$ in Algorithm~1 to get insight and intuition into why the square-root min-vol NMF formulation exhibits tuning-free behavior. For all experiments, we use an approximation parameter $\varepsilon = 0.1$.

\subsection{Success with a single tuning parameter under a range of noise levels}

The simulation setup is identical to the one given in Section~2. We seek to recover a factor matrix $\M{W}^\star \in \Real^{4 \times 4}$ and $\M{H}^\star \in \Real_+^{4 \times 500}$ in the presence of additive i.i.d.\@ Uniform$[0, \sigma]$ noise. We simulate instances of the problem with noise level $\sigma$ ranging from $10^{-10}$ to $10^{-1}$ on a log-linear scale.

Recall that the best performing tuning parameter values in standard noisy min-vol NMF were \emph{proportional with the noise level} (See \Fig{sfig2}). By contrast, for square-root min-vol NMF we set $\lambda = 1$ in \Eqn{fepsilon} in all problem realizations, i.e., \emph{we employ a single value of $\lambda$ for all noise levels.}

\Fig{minvol_MMminvol} shows that the errors of the factorization recovered by the square-root min-vol NMF tend towards zero as $\sigma$ tends to zero. Again this was accomplished without any tuning since $\lambda = 1$ for all $\sigma$ between $10^{-10}$ and $10^{-1}$.

\begin{figure}[H]
    \centering
  \includegraphics[width=.6\linewidth]{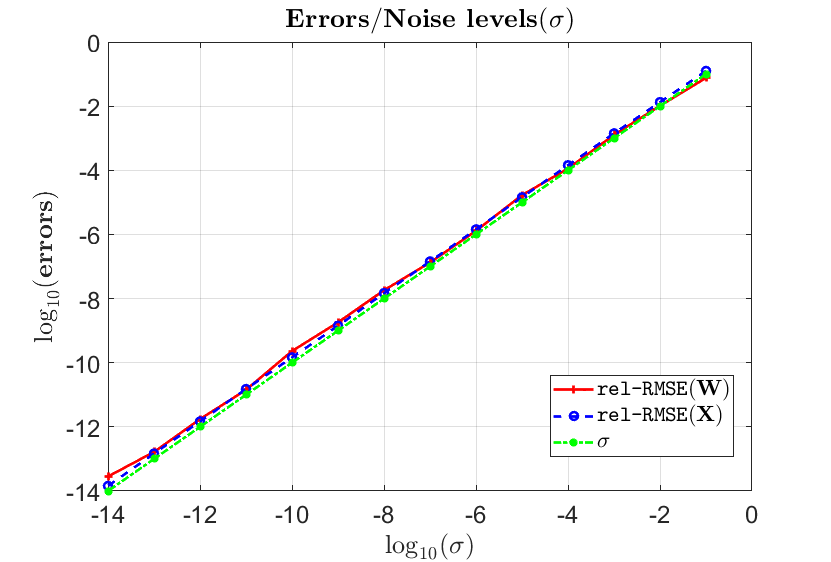}

    \caption{Relative errors of $\M{W}^\star$ and $\M{X}^\star$ recovery under varying noise levels $\sigma$ by square-root min-vol NMF using a single tuning parameter value $\lambda = 1$.}
    \label{fig:minvol_MMminvol}
\end{figure}

\subsection{Success with a wide range of tuning parameters}

In the previous experiment, we saw that a single value of the tuning parameter could be successful over a wide range of noise levels. A natural follow up question is whether the range of $\lambda$ that can lead to successful recovery is wide or narrow.  In the second experiment, we seek to answer this question by applying the square-root min-vol NMF with different $\lambda$ over a range of noise levels. We use the same simulation setup as in Section~4.1 but consider a more focused range of $\sigma \in \small{\lbrace 0, 0.1, 0.01, 0.001, 0.0001 \rbrace}$. We evaluate the recovery error using 
\begin{eqnarray*}
    \lambda & \in & \text{$\lbrace 2,1.5,1,0.8,0.5,0.1,0.05,0.01,0.005,0.001,0.0005,0.0001 \rbrace.$}
\end{eqnarray*}

To take a closer look at recovery performance, we employ principle component analysis  (PCA) to visualize the columns of $\M{X}$, $\M{W}^{\ast}$, and $\hat{\M{W}}$. 

Figures~\ref{fig:4_4_0}, \ref{fig:4_4_0-0001}, \ref{fig:4_4_0-001}, \ref{fig:4_4_0-01}, and \ref{fig:4_4_0-1} show the relative error in recovering $\M{W}^\star$ as a function of the iterations under different noise level (panel a), relative error in recovering $\M{X}$ (panel b) as a function of the iterations under different noise levels, and the PCA visualization of the columns of $\M{X}$ and estimated columns of $\M{W}^\star$ by square-root min-vol. The common trend in all plots is that for sufficiently small $\lambda$, the relative error is on the order of $\sigma$.  When $\lambda$ is too large, the recovery fails (panels a and b). Nonetheless, we see that while the efficacy of square-root min-vol depends on the choice of $\lambda$, there is a relatively wide range of $\lambda$ that leads to estimates that have relative error up to the order of $\sigma$.
\begin{figure}[H]
    \centering
    \begin{subfigure}{.33\textwidth}
  \centering
  \includegraphics[width=.95\linewidth]{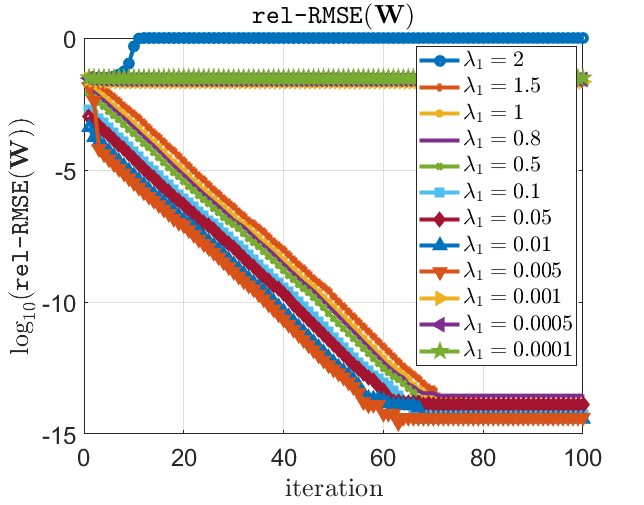}
  \caption{$ \texttt{rel-RMSE}(\M{W})$ for different $\lambda$}
  \label{fig:4_4_0_fig1}
\end{subfigure}%
\begin{subfigure}{.33\linewidth}
  \centering
  \includegraphics[width=.95\linewidth]{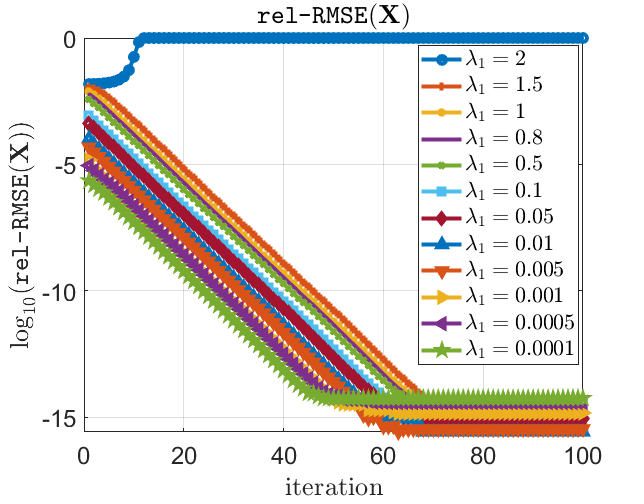}
  \caption{$ \texttt{rel-RMSE}(\M{X})$ for different $\lambda$}
  \label{fig:4_4_0_fig2}
\end{subfigure}%
\begin{subfigure}{.33\linewidth}
  \centering
  \includegraphics[width=.95\linewidth]{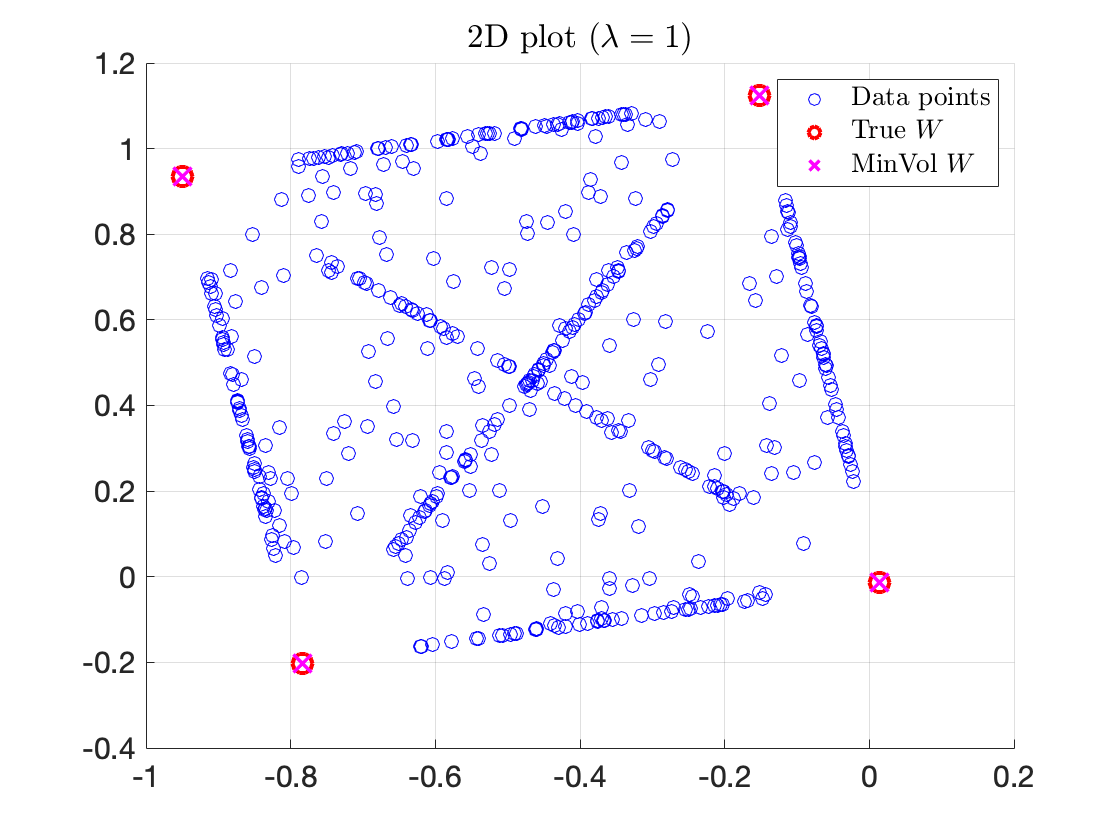}
  \caption{2D plot using PCA}
  \label{fig:4_4_0_fig3}
\end{subfigure}
    \caption{Results of square-root min-vol NMF with respect to different $\lambda$ in noiseless case}
    \label{fig:4_4_0}
\end{figure}

\begin{figure}[H]
    \centering
    \begin{subfigure}{.33\textwidth}
  \centering
  \includegraphics[width=.95\linewidth]{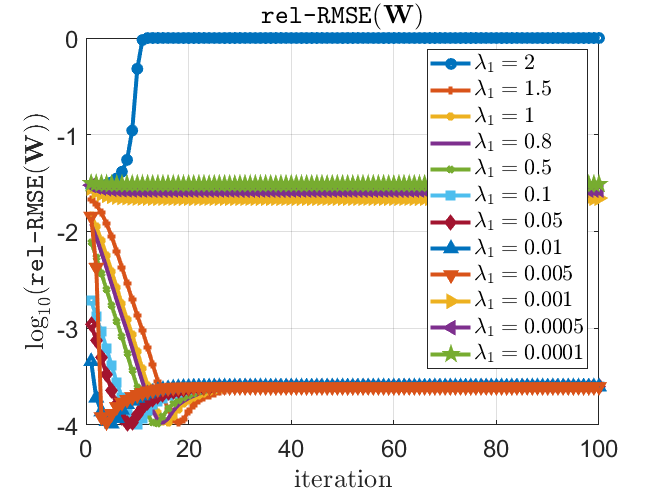}
  \caption{$ \texttt{rel-RMSE}(\M{W})$ for different $\lambda$}
  \label{fig:4_4_0-0001_fig1}
\end{subfigure}%
\begin{subfigure}{.33\textwidth}
  \centering
  \includegraphics[width=.95\linewidth]{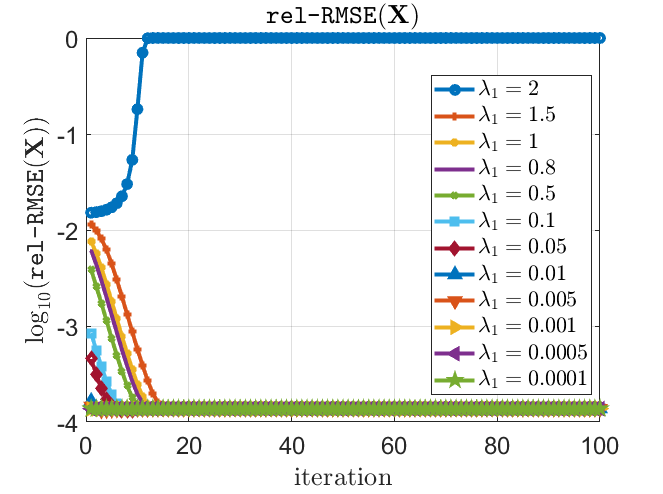}
  \caption{$ \texttt{rel-RMSE}(\M{X})$ for different $\lambda$}
  \label{fig:4_4_0-0001_fig2}
\end{subfigure}%
\begin{subfigure}{.33\textwidth}
  \centering
  \includegraphics[width=.95\linewidth]{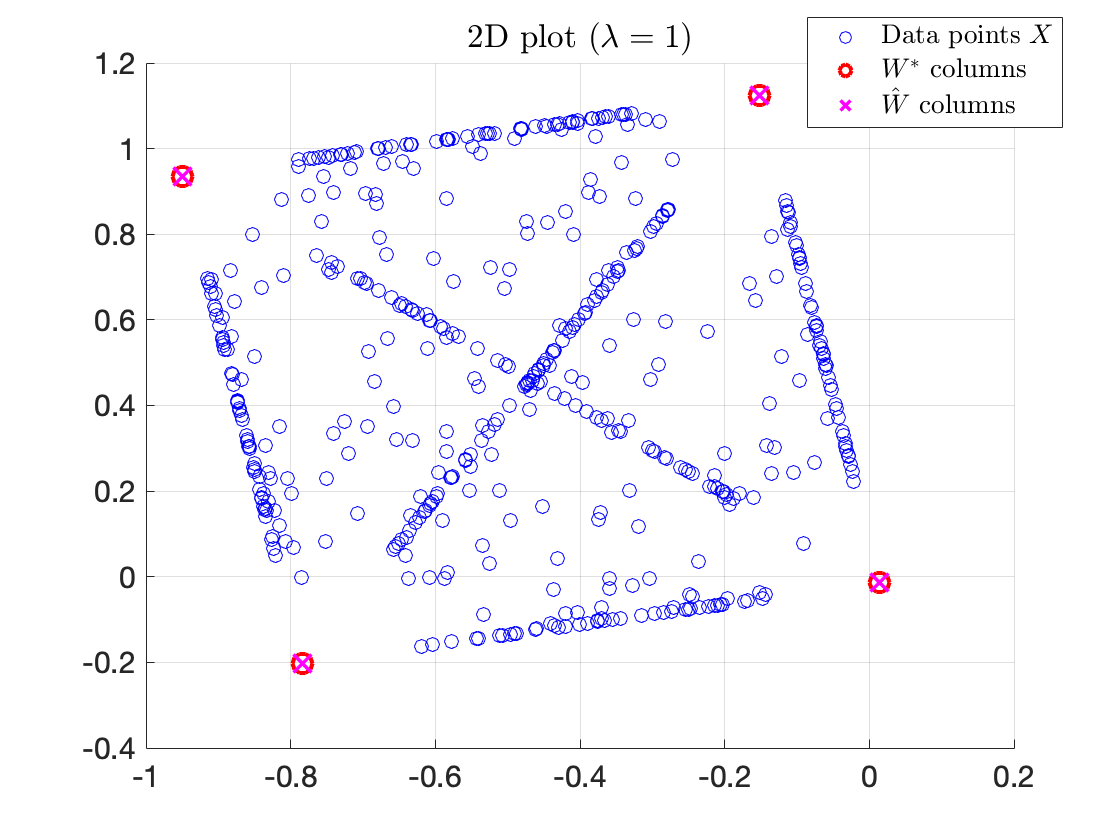}
  \caption{2D plot using PCA}
  \label{fig:4_4_0-0001_fig3}
\end{subfigure}
    \caption{Results of square-root min-vol NMF with respect to different $\lambda$ when $\sigma = 0.0001$}
    \label{fig:4_4_0-0001}
\end{figure}

\begin{figure}[H]
    \centering
    \begin{subfigure}{.33\textwidth}
  \centering
  \includegraphics[width=.95\linewidth]{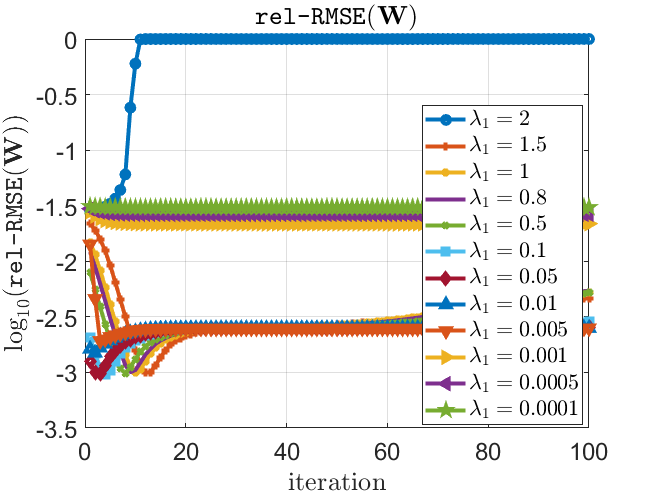}
  \caption{$ \texttt{rel-RMSE}(\M{W})$ for different $\lambda$}
  \label{fig:4_4_0-001_fig1}
\end{subfigure}%
\begin{subfigure}{.33\textwidth}
  \centering
  \includegraphics[width=.95\linewidth]{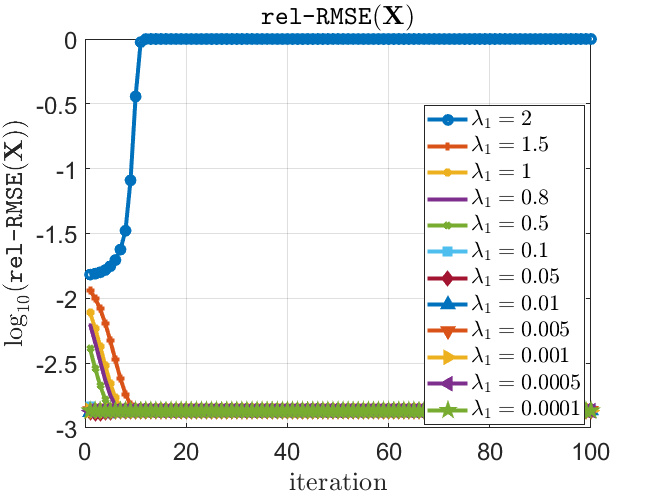}
  \caption{$ \texttt{rel-RMSE}(\M{X})$ for different $\lambda$}
  \label{fig:4_4_0-001_fig2}
\end{subfigure}%
\begin{subfigure}{.33\textwidth}
  \centering
  \includegraphics[width=.95\linewidth]{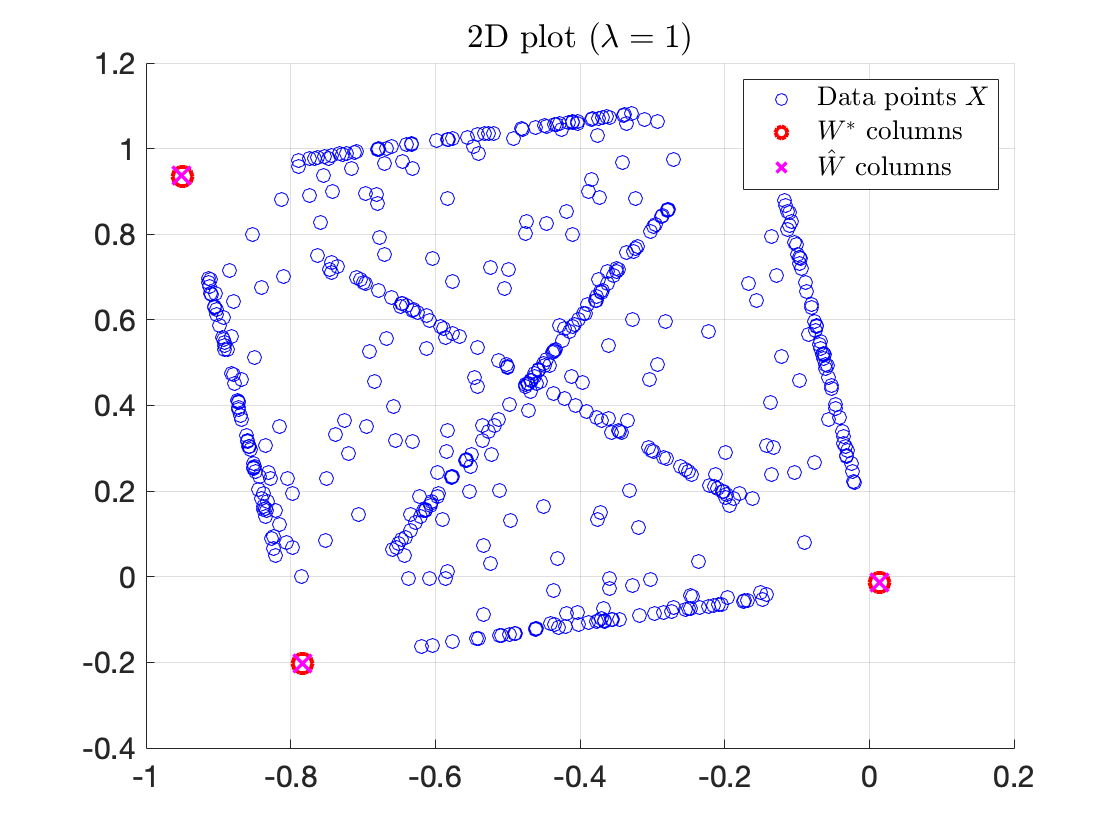}
  \caption{2D plot using PCA}
  \label{fig:4_4_0-001_fig3}
\end{subfigure}
    \caption{Results of square-root min-vol NMF with respect to different $\lambda$ when $\sigma = 0.001$}
    \label{fig:4_4_0-001}
\end{figure}

\begin{figure}[H]
    \centering
    \begin{subfigure}{.33\textwidth}
  \centering
  \includegraphics[width=\linewidth]{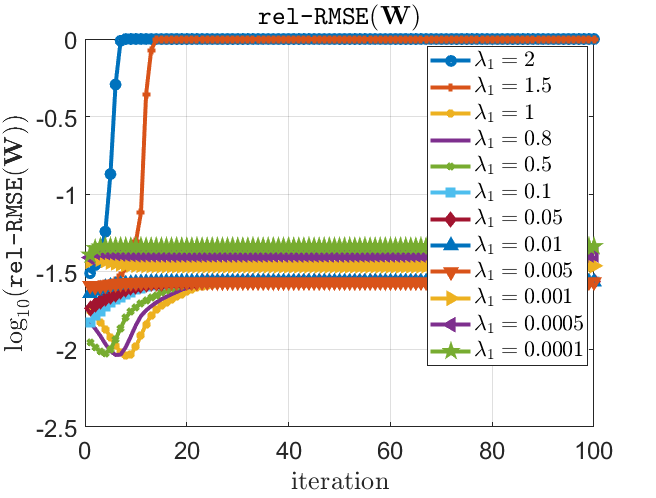}
  \caption{$ \texttt{rel-RMSE}(\M{W})$ for different $\lambda$}
  \label{fig:4_4_0-01_fig1}
\end{subfigure}%
\begin{subfigure}{.33\textwidth}
  \centering
  \includegraphics[width=\linewidth]{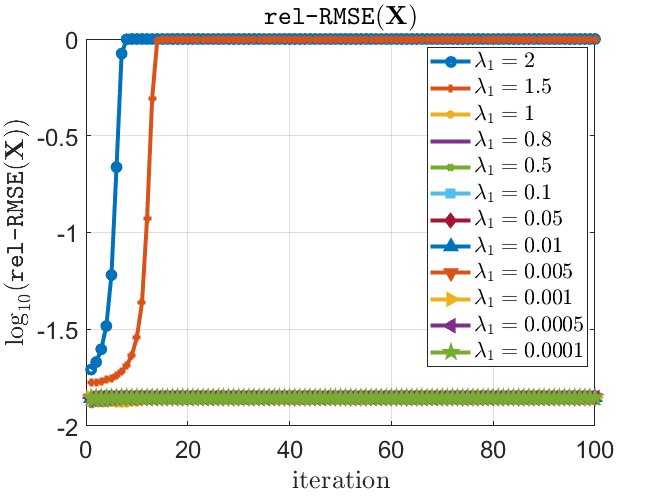}
  \caption{$ \texttt{rel-RMSE}(\M{X})$ for different $\lambda$}
  \label{fig:4_4_0-0001_fig2}
\end{subfigure}%
\begin{subfigure}{.33\textwidth}
  \centering
  \includegraphics[width=\linewidth]{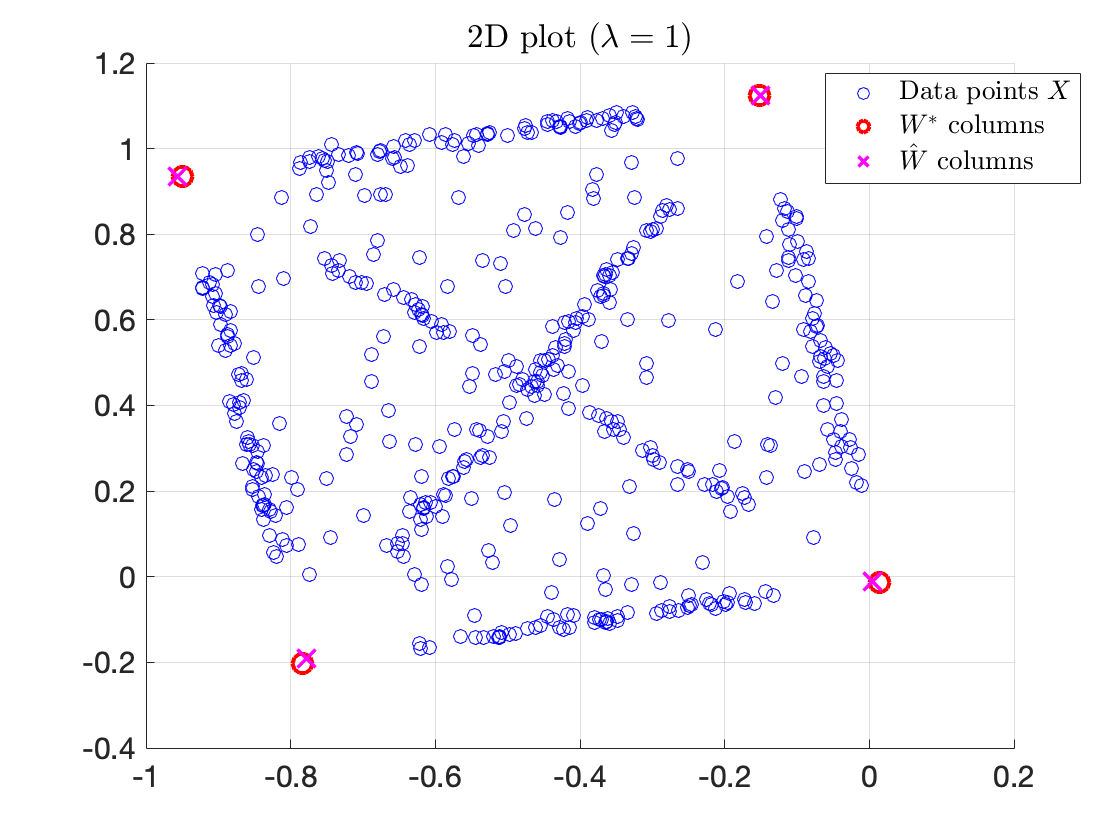}
  \caption{2D plot using PCA}
  \label{fig:4_4_0-01_fig3}
\end{subfigure}
    \caption{Results of square-root min-vol NMF with respect to different $\lambda$ when $\sigma = 0.01$}
    \label{fig:4_4_0-01}
\end{figure}

\begin{figure}[H]
    \centering
    \begin{subfigure}{.33\textwidth}
  \centering
  \includegraphics[width=.95\linewidth]{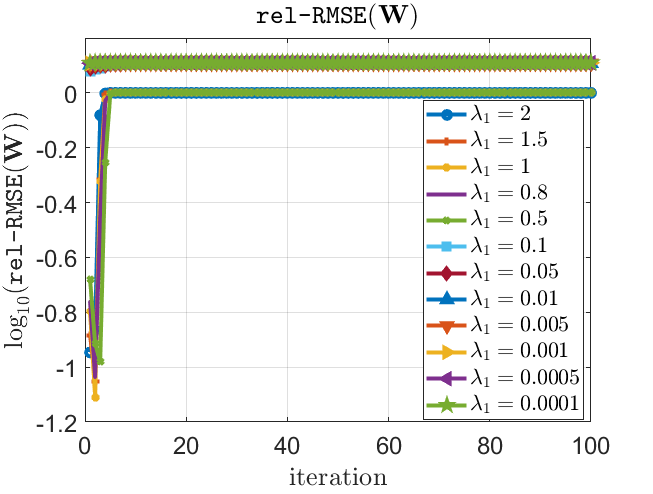}
  \caption{$ \texttt{rel-RMSE}(\M{W})$ for different $\lambda$}
  \label{fig:4_4_0-1_fig1}
\end{subfigure}%
\begin{subfigure}{.33\textwidth}
  \centering
  \includegraphics[width=.95\linewidth]{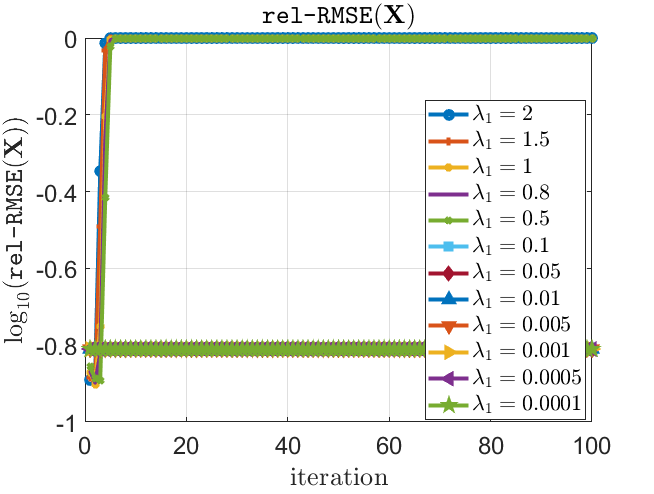}
  \caption{$ \texttt{rel-RMSE}(\M{X})$ for different $\lambda$}
  \label{fig:4_4_0-1_fig2}
\end{subfigure}%
\begin{subfigure}{.33\textwidth}
  \centering
  \includegraphics[width=.95\linewidth]{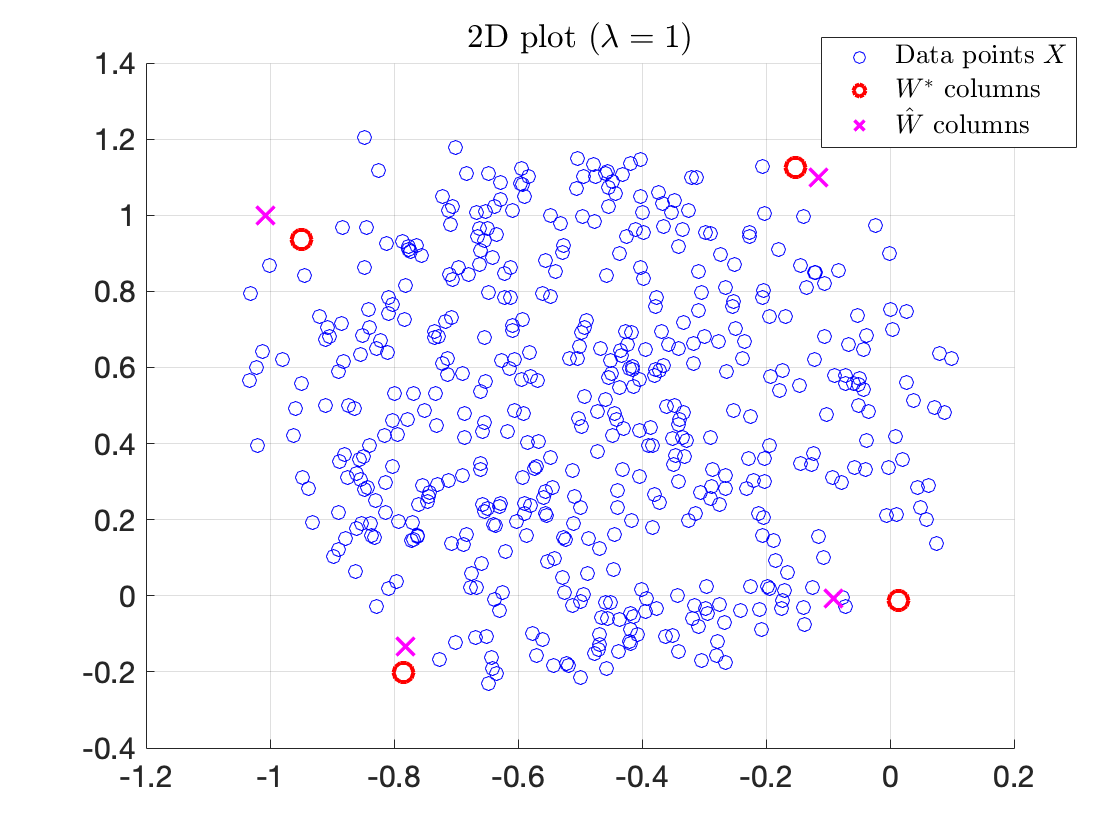}
  \caption{2D plot using PCA}
  \label{fig:4_4_0-1_fig3}
\end{subfigure}
    \caption{Results of square-root min-vol NMF with respect to different $\lambda$ when $\sigma = 0.1$}
    \label{fig:4_4_0-1}
\end{figure}


Moreover, the PCA plots (Figures~\ref{fig:4_4_0}-\ref{fig:4_4_0-1}, panel c) provide visual confirmation that the recovery of the columns of $\M{W}^\star$ is on the order of the noise level $\sigma$. The columns of $\hat{\M{W}}$ coincide exactly with the columns of $\M{W}^{\ast}$ when $\sigma = 0$. For $\sigma = 0.0001$ and $0.001$, the columns of $\M{W}^{\ast}$ are also visually indistinguishable from the columns of $\M{W}^\star$. Only when $\sigma = 0.01$ or $0.1$ can we visually discern the impact of the noise level. In those cases, the columns of $\hat{\M{W}}$ appear to be within a radius $\sigma$ of the columns of $\M{W}^{\ast}$. 

\subsection{Intuition for why the square-root min-vol NMF formulation exhibits tuning-free behavior}
In the third experiment, we track the per-iteration behavior of the MM subproblem to gain insight into why we observe this beneficial tuning-free behavior. We consider a larger and more realistic version of the simulation setup in Sections~4.1 and 4.2. We generate a random $25 \times 20$ matrix $\M{W}^{\ast}$; each entry is i.i.d.\@ Uniform[0, 1], and  $\M{H}^{\ast}$ is a $20 \times 10000$ matrix. We do not add noise to $\M{X}^\star$, and we use  $\lambda = 0.8$. 

\begin{figure}[H]
    \centering
    \begin{subfigure}{.45\linewidth}
  \centering
  \includegraphics[width=.95\textwidth]{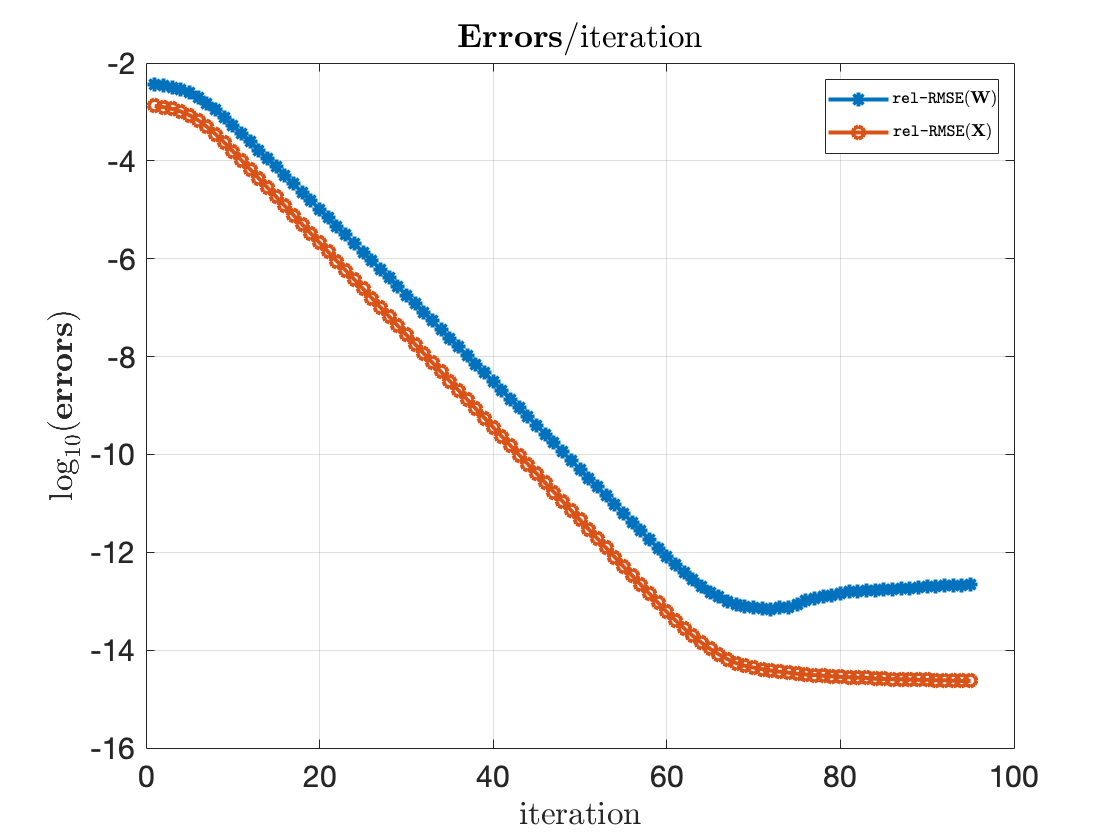}
  \caption{$ \texttt{rel-RMSE}(\M{W})$ and $ \texttt{rel-RMSE}(\M{X})$}
  \label{fig:25_20_0-1_fig1}
\end{subfigure}%
\begin{subfigure}{.45\linewidth}
  \centering
  \includegraphics[width=.95\textwidth]{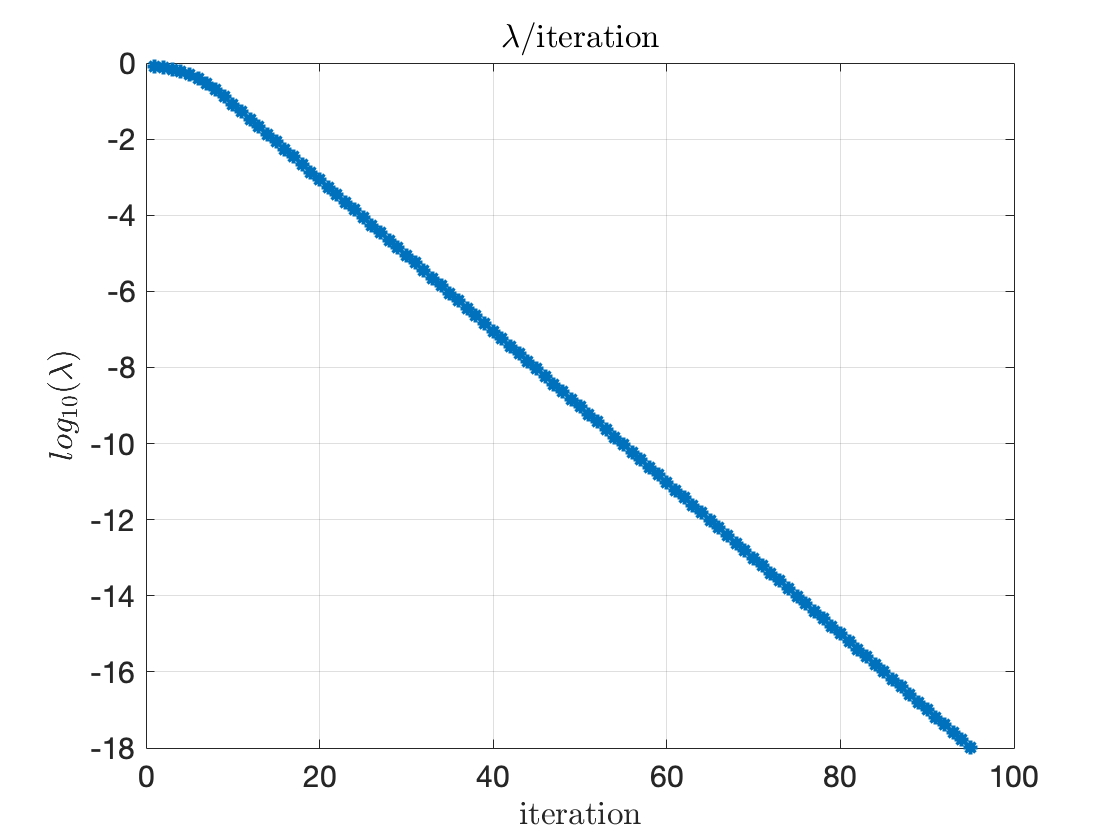}
  \caption{$\lambda_k$ vs iteration}
  \label{fig:25_20_0-1_fig2}
\end{subfigure}%
\newline
\begin{subfigure}{.5\linewidth}
  \centering
  \includegraphics[width=.95\textwidth]{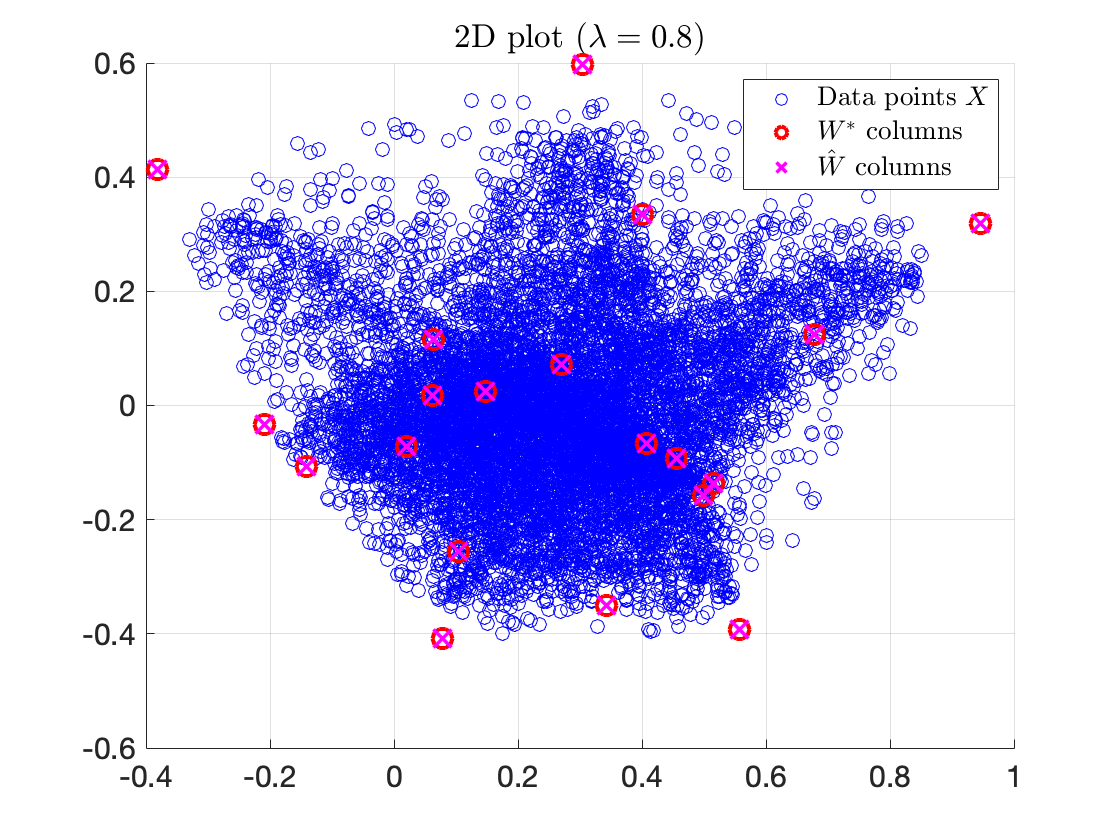}
  \caption{2D plot using PCA}
  \label{fig:25_20_0_fig3}
\end{subfigure}
    \caption{Results of square-root min-vol NMF with larger noiseless dataset}
    \label{fig:25_20_0}
\end{figure}

Figure~\ref{fig:25_20_0-1_fig1} shows that both $ \texttt{rel-RMSE}(\M{X})$ and $ \texttt{rel-RMSE}(\M{W})$ decrease to machine precision as Algorithm~1 proceeds. The 2D PCA visualization in Figure \ref{fig:25_20_0_fig3} confirms perfect recovery of the columns $\M{W}^{\ast}$ by the columns of $\Mhat{W}$.

The behavior of $\lambda_k$ as a function of the iteration count $k$ reveals an interesting insight. \Fig{4_4_0-0001} shows that $\lambda_k$ decreases rapidly to zero as $k$ increases. The key to understanding this behavior is to recognize that $\sqrt{r_k} = \sqrt{\lVert \M{X} - \M{W}_k\M{H}_k\rVert_{\text{F}}^2+\varepsilon}$ can be interpreted as an estimate of the noise level (up to some scaling based on the size of $\M{X}$) using the $k$th MM-iterate ($\M{W}_k, \M{H}_k)$. Thus, $\lambda_k = 2\sqrt{r_k}\lambda$ is a scaled version of $\lambda$ where the scaling factor is proportional to an estimate of $\sigma$ based on the $k$th MM-iterate. The MM-update is iteratively solving the surrogate optimization problem
\begin{eqnarray*}
\underset{(\M{W},\M{H}) \in \mathcal{S}}{\min} \lVert\M{X} - \M{W}\M{H}\rVert_{\text{F}}^2 + 2mn\hat{\sigma}_k\lambda\log\det\left(\M{W}\Tra\M{W} + \delta \M{I}\right),
\end{eqnarray*}
where $\hat{\sigma}_k = \frac{1}{mn} \sqrt{\lVert \M{X} - \M{W}_k\M{H}_k\rVert_{\text{F}}^2 + \varepsilon}$.
Thus, the MM-updates can be interpreted as solving the original noisy min-vol NMF problem in \Eqn{gillis_mv} with a tuning parameter that has been scaled by the most recent estimate of the noise level. Moreover, since there is no noise in this experiment, we expect that the best choice of $\lambda$ if one were to use noisy min-vol NMF would be zero, which is indeed to what $\lambda_k$ adaptively tends.



\section{Conclusion}

In this paper, we proposed a tuning-free variation on the noisy min-vol NMF model introduced in \cite{Leplat2019} as well as an MM algorithm for fitting the model that comes with some global convergence guarantees. Our empirical studies show that for sufficiently small tuning parameter values, recovery errors on the order of the added noise can be achieved. As this work is still relatively preliminary, open questions for future research include the following.
\begin{itemize}
\item Under what conditions is the square-root min-vol NMF provably guaranteed to be tuning-free in a way similar to its inspiration, the square-root lasso?
\item In this paper, we used a smooth approximation $f_\varepsilon$ and employed an MM algorithm which repeatedly applied in an inner loop the algorithm in \cite{Leplat2019}. Can we design faster algorithms for solving the square-root min-vol NMF problem? 
\end{itemize}


\section{Acknowledgement}
This work was conducted as part of the 2023 REU STAT-DATASCI program hosted by the Department of Statistics at Rice University. This work was also partially funded by a grant from the National Institute of General Medical Sciences (R01GM135928: EC).


\bibliographystyle{plain}
\bibliography{RiceREU}

\vfill\null
\vfill\null
\vfill\null
\vfill\null
\vfill\null
\vfill\null


\end{document}